\documentclass{article}

\usepackage{PRIMEarxiv} 

\usepackage[utf8]{inputenc}
\usepackage{amsmath}
\usepackage{amsthm}
\usepackage{amssymb}
\usepackage{amsfonts}
\usepackage{graphicx}
\usepackage{comment}
\usepackage{mathptmx}
\usepackage{hyperref}
\usepackage{calc}
\usepackage[nottoc]{tocbibind}
\usepackage{tikz}
\usetikzlibrary{tikzmark,calc,,arrows,shapes,decorations.pathreplacing}
\tikzset{every picture/.style={remember picture}}
\usepackage{accents}

\usepackage{braket}
\usepackage{appendix}

\newcommand\blfootnote[1]{%
  \begingroup
  \renewcommand\thefootnote{}\footnote{#1}%
  \addtocounter{footnote}{-1}%
  \endgroup
}
\newtheorem{theorem}{Theorem}[section]

\newtheorem{corollary}[theorem]{Corollary}
\newtheorem{definition}[theorem]{Definition}
\theoremstyle{definition}

\newtheorem{remark}[theorem]{Remark}

\title{A topological description of loss surfaces based on Betti Numbers}

\author{
  Maria Sofia Bucarelli$^{\dag}$\\
  DIAG \\
  Sapienza University of Rome \\
  Rome\\
  \texttt{mariasofiabucarelli@uniroma1.it} \\
  \And
  Giuseppe Alessio D'Inverno$^{\dag}$ \\
  DIISM \\
  University of Siena \\
  Siena\\
  \texttt{dinverno@diism.unisi.it} \\
  \And \\
  \And
  Monica Bianchini\\
  DIISM \\
  University of Siena \\
  Siena\\
  \texttt{monica.bianchini@unisi.it} \\
  \And 
  Franco Scarselli\\
  DIISM \\
  University of Siena \\
  Siena\\
  \texttt{franco.scarselli@unisi.it} \\
    \And 
  Fabrizio Silvestri \\
  DIAG \\
  Sapienza University of Rome \\
  Rome\\
  \texttt{fsilvestri@diag.uniroma1.it} \\
  }

\date{}

\newcounter{mathLableNode}

\begin{document}

\maketitle

\begin{abstract}
In the context of deep learning models, attention has recently been paid to studying the surface of the loss function in order to better understand training with methods based on gradient descent. This search for an appropriate description, both analytical and topological, has led to numerous efforts to identify spurious minima and characterize gradient dynamics. 
Our work aims to contribute to this field by providing a topological measure to evaluate loss complexity in the case of multilayer neural networks. We compare deep and shallow architectures with common sigmoidal activation functions by deriving upper and lower bounds on the complexity of their loss function and revealing how that complexity is influenced by the number of hidden units, training models, and the activation function used. Additionally, we found that certain variations in the loss function or model architecture, such as adding an $\ell_2$ regularization term or implementing skip connections in a feedforward network, do not affect loss topology in specific cases.

\end{abstract}

\blfootnote{$^{\dag}$ These authors equally contributed to this paper.}

\section{Introduction}\label{sec:introduction}

% In recent years the investigation on searching for an appropriate description, either in analytical or topological fancy \cite{freeman2016topology}, of the surface of the loss function in deep learning models has become a priority, in order to explain why deep neural networks are successfully trained via gradient descent based methods, despite the strong non-convexity of the associated optimization problem. In several works  remarkable efforts have been made to find either a description of the spurious minima location for specific architectures (e.g. ReLu Networks \cite{safran2018spurious,venturi2018spurious,venturi2019spurious}) or a characterization of the behaviour of gradient dynamics \cite{maennel2018gradient,williams2019gradient}. \\
% Our work is an attempt to find a topological characterization of the most used loss functions, optimized on commonly used feedforward neural networks, through a computation of tight bounds on the Betti numbers associated with the loss surface, highlighting how the number of layers can impact on this characterization. This work takes inspiration from \cite{bianchini2014complexity}, in which the same theory is exploited to prove that the geometry of the output set in shallow networks is far less complex than the one in deep networks. This intuition can better explain why the optimization task is carried on on shallow networks more easily than in the deep ones, but on the other hand it shed light over the learning capabilities of deep neural networks.

In recent years, the investigation into the theoretical foundations of Machine Learning (ML) and Deep Learning (DL) models has gained more and more attention as the research community has decided to delve deeper into the reasons why these models can achieve exceptional performance in different application fields \cite{koren2009matrix,scarselli2008graph, krizhevsky2017imagenet, devlin2018bert,radford2018improving}. On the one hand, as the automated decisions provided by these algorithms can have a relevant impact on people’s lives, their behavior has to be aligned with the values and principles of individuals and society. This demands designing automated methods we can trust, fulfilling the requirements of fairness, robustness, privacy, and explainability \cite{BiaSca22}. On the other hand, a wide range of tools arose from different areas have been taken into account to give proper explanations to the behavior of ML and DL models according to different mathematical aspects, such as the gradient descent dynamics \cite{maennel2018gradient} \cite{williams2019gradient} \cite{goodfellow2014qualitatively}, the role of the activation functions \cite{ramachandran2017searching}, and the importance of the number of layers \cite{bianchini2014complexity}. From the theoretical point of view, along with the aforementioned features, 
the characterization of the loss function to be minimized is a crucial aspect, as the whole training efficiency relies on its shape, which in turn depends on the network architecture. Several works have already dealt with the analysis of the surface of the loss function, identifying conditions for the presence (or absence) of spurious valleys in a theoretical \cite{venturi2018spurious} or empirical-driven way \cite{safran2018spurious}, pointing out the role of saddle points in slowing down the learning \cite{dauphin2014identifying}, and giving hints on the topological structure of the loss for networks with
%when minimized over networks employing 
different types of activation functions \cite{freeman2016topology,nguyen2017loss}.

Our contribution fits into the latter line of research, as we give a characterization
of the complexity of loss functions based on a topological argumentation.  
More precisely, given a layered neural network $\mathcal{N}$
and a loss function  $\mathcal{L}_{\mathcal{N}}$ computed on some training data,
we will measure the complexity of $\mathcal{L}_{\mathcal{N}}$
 by the topological complexity of the set
$S_{\mathcal{N}} = \{ \theta | \mathcal{L}_{\mathcal{N}}(\theta) \leq c \}$. 
Such an approach is natural, since $S_{\mathcal{N}}$, observed at each level $c$,
provides the form of the loss function: for example, 
if $\mathcal{L}_{\mathcal{N}}$ has $k$ isolated minima, then
$S_{\mathcal{N}}$ has $k$ disconnected regions for some small $c$.

In the paper, we will provide a bound on the sum of  Betti numbers 
\cite{bredon2013topology} of the set $S_{\mathcal{N}}$. In algebraic
topology, Betti numbers are exploited to distinguish spaces with
different topological properties. More precisely, for any subset
$S \subset \mathbb{R}^n$, there exist $n$ Betti numbers $b_i(S),0 \leq i\leq n-1$. Intuitively, the first Betti number $b_0(S)$ is the number of connected components of
the set $S$ that, in the case of $S=S_{\mathcal{N}}$, corresponds to the number of basins of attraction of the loss function. The $i$-th Betti number $b_i(S)$ counts the number of 
$(i+1)$-dimensional holes in $S$, which provides a measure of the complexity of the error 
function on a given level.

The upper bound is derived for feedforward neural networks with different numbers of layers, number of neurons per layer, and number of training samples. Moreover, we consider networks with skip connections, such as ResNet. We consider Binary Cross Entropy (BCE) and Mean Squared Error (MSE) loss functions with or without regularization. Finally, we 
treat networks with Pfaffian activation functions. The class of Pfaffian maps is broad and includes most of the functions,  with continuous derivatives, commonly used in Engineering and Computer Science applications, such as the hyperbolic tangent, the logistic sigmoid, polynomials and their compositions~\footnote{For the sake of simplicity, we do not consider networks with ReLU activation functions, which are not Pfaffian. The results of this paper can easily be extended to ReLU and, more generally, 
piece-wise polynomials, using a measure of complexity based on the number of disconnected components of the set $S_{\mathcal{N}}$. For example, such an argument has been used to estimate the  Vapnik-Chervonenkis dimension of feedforward neural networks~\cite{bartlett1998almost}. However, the
estimation requires a  different technique and a duplication of our theorems,
which here we prefer to skip.}. Most of the elementary functions are Pfaffian, for example, the exponential and trigonometric functions.   
 The concept of a Pfaffian function was first introduced in \cite{khovanskiui1991fewnomials}, where an analog of Bezout's theorem was proved. Bezout's classic theorem states that the number of complex solutions of a set of polynomial equations can be estimated based on their degree (to be precise, it is equal to the product of their degree).  The analog of this theorem proposed in \cite{khovanskiui1991fewnomials} holds for some classes of real and transcendental equations: for a wide class of real transcendental equations (including all real algebraic ones), the number of solutions of a set of $k$ such equations in $k$ real unknowns is finite, and can be explicitly estimated in terms of the "complexity" of the equations.
In the Pfaffian setting, Gabrielov and Vorobjov introduced a suitable notion of complexity or \textit{format} \cite{gabrielov2004complexity}.
We will define these concepts formally later in the Preliminaries section.

It is worth mentioning that this work takes inspiration from \cite{bianchini2014complexity}, where a similar topological argumentation has been used to study the complexity of the functions realized by a neural network, showing that deep architectures can implement more complex functions than shallow ones, using the same number of parameters.
 Thus, intuitively, while the results in \cite{bianchini2014complexity} are about
 the network function, namely what a network can do, the results in this paper
 are about the error function, namely how hard the optimization problem is.

In this work, we attempt to provide an answer to the following questions:

\begin{enumerate}
    \item  \textit{Can a topological measure effectively assess the complexity of the loss implemented by layered neural networks?} 
    \item  \textit{How do the complexity bounds of deep and shallow neural architectures relate to the number of hidden units and the selected activation function?}
\end{enumerate}

\subsubsection*{Our contribution}

 This work paves the way toward a more comprehensive understanding of the landscape of loss functions in deep learning models.
The sum of Betti numbers of the sublevel sets of the loss function is used to measure the complexity of the empirical risk landscape. More precisely, we derive upper bounds on the complexity of empirical risk for deep and shallow neural architectures employing the theory of Pfaffian functions. This study illuminates the dependence of the error surface complexity on crucial factors such as the number of hidden units, the number of training samples, and the specific choice of the activation function.

Our main contributions are summarized below.
\begin{itemize}
    \item We derive the Pfaffian format of common loss functions, i.e., the Mean Square Error (MSE) loss function and the Binary Cross-Entropy (BCE) loss function, when training feedforward networks; the Pfaffian format is computed with respect to the weights and involves the knowledge of the Pfaffian format of Pfaffian activation functions.
    \item Consequently, bounds on the complexity of such loss functions are found in terms of Betti numbers's characterization; multiple bounds are described in terms of the different parameters of the problem (i.e., number of layers, number of neurons per layer, training set size, and total number of learnable parameters). Specifically, we show (coherently with the literature and common knowledge) that the complexity of the loss function increases super-exponentially with the number of layers or neurons per layer and exponentially with the number of training samples. 
% Considerations around the possible implementation of regularizers or different architectures are drawn to give a wider perspective.
    \item We prove that adding a $\ell^2$ regularization term to the loss function does not influence its topological complexity under our theoretical analysis.
    \item  Moreover, in our study, we demonstrate that incorporating skip connections (as in ResNets \cite{he2016deep}) into the network does not affect the Betti numbers' bounds. 
\end{itemize}

The paper is organized as follows. In Section \ref{sec:related_work}, we briefly review the literature on the characterization of the loss surface and the topological tools used to evaluate the computational power and generalization ability of feedforward networks. In Section \ref{sec:preliminaries}, some notations and preliminary definitions are introduced, while in Section
\ref{sec:main_results}, the main results proposed in this paper are presented. To ensure easy text reading, which allows one to grasp its main contents, all proofs are collected in the Appendix. Finally, some conclusions and future perspectives are reported in Section \ref{sec:conclusions}.

\section{Related Work}
\label{sec:related_work}
The characterization of the loss landscape has gained more and more attention in the last few years; attempts to provide a satisfying description have been made through several approaches. In~\cite{choromanska2015loss,choromanska2015open}, the spin-glass theory is exploited to quantify, in probability, the presence of local minima and saddle points. Unfortunately, the connection between the loss function of neural networks and the Hamiltonian of the spherical
spin-glass models relies on a number of possibly unrealistic assumptions, yet the empirical evidence suggests that it may exist also under mild conditions. In~\cite{dauphin2014identifying}, a method that can rapidly escape high dimensional saddle points, unlike gradient descent and quasi-Newton methods, is proposed.
Based on results from statistical physics, random matrix theory, neural network theory, and empirical evidence, it is argued that the major difficulty in using local optimization methods originates from the proliferation of saddle points, instead of local minima, especially in high dimensional problems of practical interest. Saddle points, in fact,  are surrounded by plateaus that can dramatically slow down the learning process. In \cite{venturi2019spurious}, the topological property of the loss function defined as \textit{presence or absence of spurious valleys} (i.e., local minima) is addressed for one-hidden layer neural networks, providing the following contributions:
% \ad{1) the Empirical Risk Minimization loss for any continuous activation function does not exhibit spurious valleys as long as the network is sufficiently over-parametrised. The same applies for the expected value loss with polynomial activations; 2) 
% %nevertheless, being  $\text{dim}_*$ the \textit{lower intrinsic dimension}, 
% if the width of the network is smaller than the \textit{lower intrinsic dimension}, defined as the minimum number of hidden units to describe any element of the functional
% space, then there exists a distribution s.t. the square loss function $\mathcal{L}$ admits spurious valleys. In this case, spurious valleys come up over generic architectures for non-polynomial non-negative activations, among which ReLU networks are included.}
% \textcolor{teal}{Io metterei i loro punti 1 e 2, magari rephrased: 
1) the Empirical Risk Minimization loss for any continuous activation function and the Expected Value loss with polynomial activations do not exhibit spurious valleys as long as the network is sufficiently over-parametrized; 
2) for non-polynomial and non-negative activations, for any hidden width, there exists a data distribution that produces spurious valleys (with non-zero dimension), whose value is arbitrarily far from that of the global minimum;
3) finally, drawing on connections with random feature expansions, even if
spurious valleys may appear in general their measure decreases as their width increases.
This holds up to a low energy threshold, approaching the global minimum at a
speed inversely proportional to the size of the hidden layer.

Based on computer-driven empirical proof, similar results on the characterization of the loss surface in terms of the presence of spurious valleys are also reported in \cite{safran2018spurious}. In \cite{freeman2019topology}, the loss surface is studied in terms of level sets, and it is shown that the landscape of the loss functions of deep  networks with linear activation functions is significantly different from that exploiting  half-rectified ones:
in the absence of nonlinearities, the level sets are connected, while, in the half-rectified case, the topology is intrinsically different and clearly dependent on
the interplay between data distribution and model architecture. Finally, a comprehensive research survey focused on analyzing the loss landscape can be found in~\cite{berner2021modern}.

In algebraic topology, Betti numbers are used to distinguish spaces with
different topological properties. 
Betti numbers have been exploited either to give a topological description of the complexity of the function implemented by neural networks with Pfaffian activation functions or to describe their generalization capabilities. 
More specifically, in~\cite{karpinski1997polynomial}, by such a technique,  bounds on the VC dimension of feedforward neural networks are provided; this work has been later extended to  Recurrent and Graph Neural Networks~\cite{scarselli2018vapnik}. In~\cite{bianchini2014complexity}, bounds on the sum of Betti numbers are provided in order to describe the complexity of the map implemented by feedforward networks with Pfaffian activation functions. 
In \cite{Naitzat22}, the relative efficacy of ReLUs over traditional sigmoidal activations is justified based on the different speeds with which they change the topology of a dataset --- as it passes through the layers of a well-trained neural network~\footnote{This corresponds to a network with perfect accuracy on its training set and a near-zero generalization error.} --- representing two classes
of objects in a binary classification problem. This dataset is viewed as a combination of two components: the first component represents the topological manifold of elements from the first class, and the second component encompasses the elements from the second class.
A ReLU-activated neural network (neither a homeomorphism nor Pfaffian) can sharply reduce Betti numbers of the two components, but not a sigmoidal-activated one (which is a homomorphism).  Reducing the Betti numbers means that the neural network simplifies the structure of the dataset by reducing the number of connected components, holes, or voids within the data manifold. This reduction suggests that the network is indeed simplifying or transforming the dataset topology, making it more amenable for analysis and classification. 
Finally, this research suggests that, when dealing with higher topological complexity data (meaning the data has more intricate or complex shapes and relationships), we need neural networks with greater depth (more layers) to adequately capture and understand these complexities.

% while in \cite{Naitzat22} the relative efficacy of ReLUs over traditional sigmoidal activations is justified based on the different speeds with which they change the \textit{data topology}  \note{cosa vuol dire change data topology?} —-- a ReLU-activated neural network (which is neither a homeomorphism nor Pfaffian) is able to sharply reduce Betti numbers but not a sigmoidal-activated one (which is a homeomorphism). It is also proved that the higher the topological complexity of the data, the greater the depth of the network required to reduce it, explaining the need for having an adequate number of layers.

 %Adopting the same analysis tools exploited in ~\cite{bianchini2014complexity}, our work will focus on the analysis of the loss landscape associated with feedforward neural networks. 
 %\note{Questa frase o la espandiamo dando più dettagli sulle differnze etc o la togliamo. 
%}
\section{Preliminaries}
\label{sec:preliminaries}
In this section, we introduce the notation, basic concepts, and definitions, which will be functional to the subsequent description of the main results.
In the following, we will deal exclusively with feedforward neural networks, whose definition is introduced as follows.

\paragraph{Feedforward neural networks ---}
Let $\theta = \{ \tilde{W}^1, b^1, \dots, \tilde{W}^L, b^L \}$ be the set of the trainable network parameters, where $\tilde{W}^l \in \mathbb{R}^{n_l \times n_{l-1}}$, $b^l \in \mathbb{R}^{n_l \times 1}$, $l=1, \ldots, L $ identifies each layer, and $n_0, \dots , n_{L} \in \mathbb{N}$ denote the number of neurons per layer. In the following,
we assume, without loss of generality, that the network has a single output, i.e., $n_L = 1$. Note that the last assumption is unnecessary to demonstrate the Pfaffian nature of the loss function. However, its inclusion significantly simplifies the subsequent calculations involved in the analysis. The total number of parameters is $\tilde{n} = \sum \limits_{l=1}^{L} n_l (n_{l-1}+1) $.

Let $x\in \mathbb{R}^{n_0}$ be the input to the neural network. The function implemented by  the layered network is $f_{\theta}(x): \mathbb{R}^{n_0} \rightarrow \mathbb{R}$, where $f_{\theta}(x) = g(\tilde{W}^L \sigma(\tilde{W}^{L-1} \cdots \sigma(\tilde{W}^1 x + b^1)) \cdots +b^{L-1} )+ b^L) $, where $\sigma$ is the hidden layer activation function and $g$ is the activation function of the  output neuron.
%\fs{Usiamo un simbolo diverso per parametri e parametri aggregati. Ad es w piccolo per quelli non aggregati}

To simplify the notation, we will aggregate the weights and bias from the same layer into a single matrix, the \textit{augmented weight matrix} ${W}^l= [b^l, \tilde{W}^l]$; moreover, we will denote by $z^{l}$ the output of the $l$-th layer and by $a^l$ the neuron activations at the same layer. Thus, we have

\begin{align*}
z^0 = \begin{bmatrix} 1 \\ x
\end{bmatrix}, && a^l = {W}^l z^{l-1}, && z^l = \begin{bmatrix} 1 \\ \sigma(a^l) 
\end{bmatrix}, &&
\text{ for } 1 \leq l \leq L\,,
\end{align*}
and $f_{\theta}(x) = g(a^L)$.

\paragraph{Loss functions ---} Let $D = (x_i,y_i)_{i=1}^m$ be a set of training data, with $x_i \in \mathbb{R}^{n_0}$ and $y_i \in \mathbb{R}$. Let $\mathcal{L}$ be a  generic loss to be minimized over the parameters $\theta$.
 The empirical risk of loss (or cost) function is  defined as the average of the per-sample contributions, where each sample contribution is a measure of the error between the network output and the target value for that sample: %Denoting by $m$ the number of samples used for training, this can be expressed as
 $$ \mathcal{L}(\theta, D)= \frac{1}{m} \sum_{i=1}^m  loss(f_{\theta}(x_i),y_i),$$
being $y_i$ the target for the $i$-the pattern $x_i$. 
In this work we will study the topological complexity of the landscape  of the Mean Square Error (MSE) and Binary Cross-Entropy (BCE) loss functions, which are  defined as

\begin{align*}
\mathcal{L}_{\text{MSE}}(\theta,D)= \frac{1}{m} \sum_{i=1}^m  (y_i - f_{\theta}(x_i)  )^2,  && \mathcal{L}_{\text{BCE}}(\theta,D) = \sum_{i=1}^m - y_i  \log(f_{\theta}(x_i)) - (1-y_i)  \log(1- f_{\theta}(x_i)).
\end{align*}

In the following, we will remove the dependency on the dataset $D$ in the notation of empirical risk. We will focus solely on its reliance on the network parameters $\theta$. Indeed, we want to study the topological complexity of the sublevel set of the empirical risk as a function of the parameters of the network, considering the dataset samples as fixed. Unless specified otherwise, the notation $\mathcal{L}(\theta)$ in the following will represent the empirical risk corresponding to the loss function $\mathcal{L}$ applied to a dataset containing $m$ pairs $(x_i, y_i)$.
In general, we consider the targets to be real numbers.
Notice that since we're computing the topological complexity of the empirical risk, the complexity obtained using the Pfaffian format will also depend on the number of samples in $D$. 
However, for classification problems where the Binary Cross-Entropy (BCE) loss is used, each target $y_i$ is either $0$ or $1$.

Moreover, we consider a regularized form for the objective function that can be expressed as

\begin{equation*}
    \Tilde{\mathcal{L}}(\theta) = \mathcal{L}(\theta) + \lambda \Omega(\theta),
\end{equation*}
\noindent
where $\Omega(\theta)$ is a regularization term, for instance, $\Omega(\theta)= \frac{1}{2}\|w \|_2^2 $ ($\ell^2$ regularization norm).

\paragraph{Pfaffian Functions ---} Pfaffian functions \cite{Khovanskii91} are analytic functions satisfying triangular systems of first order partial differential equations with
polynomial coefficients.
For this kind of functions, an analogous of the Bézout theorem holds.
The classical Bézout theorem states that the number of complex
solutions of a set of $k$  polynomial equations in $k$ unknowns can be estimated in terms of their degrees (it equals the product of the degrees).

For a wide class of real transcendental equations (including all real algebraic ones) the number of solutions of a set of $k$ such equations in $k$ real unknowns is finite and can be explicitly estimated in
terms of the `` complexity'' of the equations, which leads to the version of Bézout theorem for Pfaffian curves and Pfaffian manifold \cite{Khovanskii91}.

The class of Pfaffian functions includes a wide variety of known functions, e.g. the elementary functions, including exponential, logarithm, tangent, and their combinations
\cite{gabrielov2004complexity}.  Interestingly, many common activation functions used in neural networks, such as the sigmoid function and hyperbolic tangent, can be classified as Pfaffian functions. 
Intuitively, a function is Pfaffian if its derivatives can be defined in terms of
polynomials of the original function and/or a chain of other Pfaffian functions. Formally, we can state the following definition.
\begin{definition}
A \emph{Pfaffian chain} of order $\ell\geq 0$ and degree $\alpha\geq 1$, in an open domain $U\subseteq\mathbb{R}^n$, is a sequence of real analytic functions $f_1,f_2,\ldots,f_\ell$, defined on $U$, satisfying the differential equations
$$
\frac{\partial f_i }{\partial x_j}(x)= P_{ij}(x,f_1(x),\ldots,f_i(x)),
$$
for $1\leq i,j \leq\ell$ and $x = (x_1, \dots, x_n) \in U $. Here, $P_{ij}(x,y_1,\ldots,y_i)$ are polynomials in the $n + i$ variables $x_1,\ldots,x_n, y_1,\ldots,y_i$ of degree not exceeding $\alpha$. \end{definition}
\begin{definition}
    Let $(f_1, \dots , f_{\ell})$ be a Pfaffian chain of length $\ell$ and degree $\alpha$, and let $U$ be its domain. A function $f$ defined on $U$ is called a Pfaffian function of degree $\beta $ in the chain $(f_1, \dots , f_{\ell})$ if there exists a polynomial $P$ in $n+ \ell$ variables, of degree at most $\beta$, such that   $f(x)=P(x,f_1(x),\ldots,f_\ell(x))$, $ \forall x \in U$.
  The triple $(\alpha,\beta,\ell)$ is called the format of $f$.
\end{definition}
%In other words, a Pfaffian function $f(x)$ is a polynomial in $x$ and in the functions appearing in a Pfaffian chain. 

The polynomial $P_{ij}$ itself may explicitly depend on $x$. Moreover, even if $P_{ij}$ does not have a direct dependence on $x$, it can still depend on $x$ indirectly through the function $f_1(x)$. We say that the polynomial $P_{ij}$ depends directly on $x$ if there are occurrences of $x$  that are not of the type $f_i(x)$.
For example, consider the function $f(x) = \arctan(x)$, which falls into the second category of functions described above since it can be represented by the chain $(f_1, f_2)$, where $f_2(x) = \arctan(x)$ and $f_1(x) = (1+x^2)^{-1}$. In this case, $\frac{\partial f_2}{\partial x} = P_1(x,f_1) = f_1(x)$, and $\frac{\partial f_1}{\partial x} = P_2(x,f_1) = x f_1(x)^2$, with $P_2$ that explicitly  depend on $x$.
On the other hand, the $\tanh$ function is a Pfaffian function falling under the first category. It can be represented by the chain containing only $f_1(x) = \tanh(x)$, and $\frac{\partial f_1(x)}{\partial x} = P_1(x) = 1 - f_1(x)^2$. In this case, there is no explicit dependence on $x$ in the polynomial $P_1$.
This distinction will be crucial to assess the right computations in the statement of Theorem \ref{th:mse_loss}.

Let us now introduce the notion of Pfaffian variety and semi-Pfaffian variety.
\begin{definition}
    The set $V \subset U$ is a Pfaffian variety if there are Pfaffian functions $p_1, \dots, p_r$ in the chain $(\text{f}_1, \dots , \text{f}_{\ell})$ such that $V= \{ x \in U : p_1(x) = \dots = p_r(x) =0\}$.
\end{definition}

\begin{definition}
A basic semi-variety $S$ on the variety $V$ is a subset of $V$ defined by a set of sign
conditions (inequalities or equalities) based on the Pfaffian functions $p_1, \dots p_s$ in the chain $(f_, \dots, f_{\ell})$ such that
$S = \{x \in V : p_1(x)\varepsilon_1 0 \And \dots \And p_s(x)\varepsilon_s 0\}$,
where $\varepsilon_1, \dots, \varepsilon_s$ are any comparison operator among $\{<; >; \leq , \geq; = \}$.
%A semi-Pfaffian set on $V$ is any  finite union of basic semi-Pfaffian sets on $V$. 
\end{definition}

% \note{occorre definir formalmente lo spazio che si studia, cioè lo spazio dei parametri dove la funzione errore è minore di una certa soglia Riprendiamo la definizione dall íntroduzione e ripetiamo ùi con qualche parola diversa/ in più}
As outlined in the Introduction, our focus in this work is directed towards exploring the complexity of the topological space defined by the sublevel set of the empirical risk related to the loss function $\mathcal{L}$, namely $S = \{ \theta : \mathcal{L}(\theta) \leq c \}$. This set represents the collection of parameter values $\theta$ for which the empirical risk of the loss function is less than or equal to a constant $c$.
When the empirical risk of the loss function is a Pfaffian function with respect to the parameter of the network, this set is exactly a Pfaffian semi-variety. 
Level curves or basins of attractions can be often described in terms of Pfaffian varieties, whose complexity can be measured through the characterization of their Betti numbers \cite{bianchini2014complexity} or counting directly the number of connected components \cite{gabrielov2004complexity}.

\paragraph{Betti Numbers ---}
Betti numbers are topological objects that can be used to describe the complexity of topological spaces.
More formally, the $i$-th Betti number of a space $X$ is defined as the rank of the (finitely generated) $i$-th singular homology group of $X$. 
Roughly speaking, it counts the number of $i$-th dimensional holes of a space $X$ \footnote{A $i$-th dimensional hole is a $i$-dimensional cycle that is not a boundary of a $(i+1)$-dimensional manifold.} and captures a topological notion of complexity that can be used to compare subspaces of $\mathbb{R}^d$. The reader can refer to algebraic topology textbooks for a more comprehensive introduction to homology \cite{bredon2013topology, hatcher2002algebraic}.

Informally, Betti numbers quantify the number of ``holes'' of various dimensions in a topological space. Each Betti number, $b_i$, represents the number of i-dimensional 'independent' holes or cycles that cannot be continuously deformed into each other. For instance, the 0-th Betti number, $b_0$, counts the number of connected components in the space.
The first Betti number, $b_1$, counts the number of independent loops or one-dimensional holes. Higher Betti numbers, such as $b_2, b_3$, and so on, count holes of increasing dimensionality.
To give some examples, 
a circle has one connected component ($b_0=1$) and one hole of dimension one ($b_1=1$).
On the contrary, a 2-dimensional unit sphere, $ \{x \in \mathbb{R}^2: ||x||_2 = 1 \}$, has one connected component ($b_0=1$) and no  holes  ($b_1=0$) but has one two-dimensional cavity ($b_2 =1$) .
A torus, like a doughnut, has one connected component ($b_0=1$) two independent holes ($b_1=2$), and one two-dimensional cavity ($b_2=1$). This distinction can be understood visually: a sphere is a solid shape with no internal holes, while a torus has a hole in its center and an additional loop around the hole. When applied to the context of a loss function, Betti numbers offer insights into the complexity of its loss landscape, such as the presence of multiple local minima and regions of attraction.
In other words, Betti numbers can be considered a tool to analyze the configuration of critical points in the optimization function landscape. %\fs{Non mi torna il fatto che si possano riconoscerei flessi. Forse le zone piatte se si considera le curve di livello cioè ${\cal L }=0$}

The following result connects the theory of Pfaffian functions and Betti numbers; in particular, it gives a bound on the Betti numbers for varieties defined by equations, including Pfaffian functions.

\begin{theorem}
\label{Teo:8}[Sum of the Betti numbers for a Pfaffian variety \cite{Zell1999}]
Let $S$ be a compact semi-Pfaffian variety in  $U \subset \mathbb{R}^{\tilde{n}}$, given on a compact Pfaffian variety $V$, of dimension $n'$, defined by $s$ sign conditions of Pfaffian functions. If all the functions defining $S$ have complexity at most $(\alpha, \beta,\ell)$, then the sum of the Betti numbers of $S$ is

\begin{equation} 
\label{eq:Bettinumber}
B(S) \in s^{n'} 2^{(\ell(\ell-1))/2}O((\tilde{n} \beta + \min(\tilde{n},\ell)\alpha)^{\tilde{n}+\ell}). \end{equation}

\end{theorem}
In this paper, the theorem is applied on the Pfaffian variety $S_{\cal N}=\{\theta \in U, \text{s.t.}\,\, \mathcal{L}(\theta)\leq c\}$, determined by the unique sign condition $ \mathcal{L}(\theta)\leq c$, for a given threshold $c$. In this way, we will obtain a bound on the sum of the Betti numbers of $S_{\cal N}$.
Therefore, we have to demonstrate that the loss function is a Pfaffian function
and compute its format,
a goal that can be achieved by writing the loss derivates in terms of the network
parameters and a Pfaffian chain.
%in with respect to a chain to be defined.
%Given the Pfaffian condition, to characterize the Pfaffian chain for the loss function, we have to study its derivatives with respect to the weights.
 We will make use of the chain rule of Backpropagation to write the derivatives of the loss function.
%derivatives of the loss function with respect to the parameters in $\theta$.

%\fs{Quì sembra che si voglia spiegare backpropagation. Questa parte penso si possa spostare nella dimostrazione, in quanto può apparire ovvio (per chi si occupa di reti neurali ) e non serve a comprendere i risultati dopo.}

 The constraints on the compactness of $U$ and $V$ can be removed without affecting the bounds, as shown in \cite{zell2003quantitative}.
In our case, we will consider $U = \mathbb{R}^n$, since it represents the domain of the parameters, moreover, given that the semi-Pfaffian variety is defined by a unique sign condition, $s=1$.

\section{Main Results}\label{sec:main_results}

In this section, we present our theoretical analysis on the loss landscape topology. We start proving that, given a  Pfaffian activation function $\sigma$ of format $ ( \alpha_{\sigma}, \beta_{\sigma}, \ell_{\sigma})$, the MSE loss function and BCE loss function computed over feedforward neural networks are also Pfaffian functions; their format is provided with an explicit dependency on the format of $\sigma$, on the number of layers, and the number of neurons per layer.

\begin{theorem}[MSE Loss] \label{th:mse_loss}
Let $\mathbf{\sigma}: \mathbb{R} \to \mathbb{R}$ be a function for which exists a Pfaffian chain $ (\mathbf{\sigma}_1, \dots, \mathbf{\sigma}_\ell)$ and $\ell_{\sigma} + 1$ polynomials $Q$ and $P_i$, $1 \leq i \leq \ell_{\sigma}$ of degree $\beta_{\sigma} $ and $\alpha_{\sigma}$, respectively s.t. $\sigma$ is Pfaffian with format $(\alpha_{\sigma}, \beta_{\sigma}, \ell_{\sigma})$.

Moreover, let $g: \mathbb{R} \rightarrow \mathbb{R}$ be a function for which there exists a Pfaffian chain $( g_1, \dots, g_{\ell_{g}})$ and $\ell_{g} +1$ polynomials,  $Q_g$  and $P_g^i$, $1 \leq i \leq \ell_g$ of degree $\beta_g$ and $\alpha_g$, respectively, s.t. $g$ is Pfaffian with format $(\alpha_g, \beta_g, \ell_g)$.

% $$ \frac{\partial g_i(a)}{\partial a}  = P_g^i(a, g_1(a), \dots, g_i(a)), \; \; 1 \leq i \leq \ell_{g}$$

% $$ g(a) = Q_g(g_1(a), \dots, g_{\ell_{g}} (a)) $$

% \fs{Penso che c'è nel testo del teorema fino si possa dire molto più semplicemente
% dicendo le funzioni di attivazione sono pfaffiane con un certo
% formato. Poi quì si  esplicita anche cosa vuol dire usando la definizione di Pfaffiana, ma questo non serve. }

%Let $f(\theta,x )$ be the function implemented by a neural network with parameters $\theta \in \mathbb{R}^n$ and input $x \in \mathbb{R}^{n_0}$.
Let $f(\theta,x )$ be the function implemented by a neural network with parameters $\theta \in \mathbb{R}^{ \tilde{n}}$, input $x \in \mathbb{R}^{n_0}$, $L$ layers, and an activation function $\sigma$ for all layers except the last. The last layer can either have an activation function $g$ or be linear.

%Let us consider the case of a real output, i.e. $f : \mathbb{R}^n \rightarrow \mathbb{R}$.

Then, the MSE Loss function is Pfaffian with format 
%\textcolor{red}{$ ((L-2)(\beta + \alpha + \alpha(\beta +1)) + \beta , \beta +2, (\ell h  (L-1))m $}

$$\left( (\text{degree}(\sigma') +1 ) (L-2) +\text{degree}(\sigma') , \; 2(\beta_{\sigma} +1), \; m \ell_{\sigma}\sum_{k=1}^{L-1} n_k\right)$$

when the last layer is linear. Here, 

 $$
\text{degree}(\sigma')= \begin{cases}
    \beta_{\sigma}  + \alpha_{\sigma} - 1  & \text{case } 1\\ 
   \beta_{\sigma}  + \alpha_{\sigma} - 1 + \alpha_{\sigma} (\beta_{\sigma} +1)  & \text{case  }2
\end{cases} $$

where case $1$ refers to the case where $P^i_{\sigma}(a, \sigma_1(a), \dots , \sigma_{\ell_{\sigma}}(a))$ 
do not depend explicitly on $a$, namely occurrences of $a$ appear that are not of the type $\sigma_i(a)$; case $2$ refers to the case where they depend on it. Moreover, we assume that the polynomials $P^i_g(a, g_1(a), \dots,g_{\ell_g}(a))$ do not depend explicitly on $a$.

The format of the chain becomes 
%\textcolor{red}{$ ((L-2) ( \beta + \alpha + \alpha (\beta+1)) +\beta +1, \; \; \beta +1 , \; \;  (\ell h ( L- 1) + \ell_{o}    )m )$}
$$\left((\text{degree}(\sigma') +1 ) (L-2) +\text{degree}(\sigma') + \text{degree}(g')+1, 2 \beta_g,  m (\ell_{\sigma}\sum_{k=1}^{L-1} n_k + \ell_{g} )\right),$$ 

if the nonlinearity $g$ is applied also to the last layer.
The definition of $\text{degree}(g')$ is analogous to that of $\text{degree}(\sigma')$.
\end{theorem}
% \note{Dobbiamo spiegare qual'è il razionale per cui assmiamo che  polinomi non dipendano in $a$. Non era più semplice non avere questa assunzione ed eventualmente specializzzare dopo i risultati? Inoltre, può convenire indicare esplicatamente 
% quali sono le chain delle Pfaffiane. }

% \note{C'è qualcosa che non mi torna nei bound. Perchè sparisce il termine
% $2\beta_{\sigma} $ nel caso uscita non lineare? La chain in quel caso è un'estensione
% del caso lineare e l '$\alpha$ non può diminuire. 
% }

% \ad{$\alpha$ non diminuisce difatti; mentre $\beta$ diventa più piccolo perchè inseriamo più termini nella catena, e questo fa sì che siano proprio i termini inseriti formino il polinomio che però ha magari grado minore}

%\fs{Il paragrafo sopra ha bisogno di aggiustamenti e maggiori spiegazioni.}

%\fs{Non è chiaro cosa vuol dire  $P^i_{\sigma}(a, \sigma_1(a), \dots , \sigma_{\ell_{\sigma}}(a))$  non dipende explicitamente da $a$, perchè non è chiaro cosa voglia dire esplicitamente. Forse si vuol dire che nel polinomio non compaiono occorenze di $a$ che non siano del tipo $\sigma_i(a)$? Inoltre, occorre spiegare con esempi quali sono le funzioni di attivazione che corrispondono a caso 1 e 2. Ad esempio, $tanh$ e $atan$ hanno questa caratteristica.}

For the BCE Loss function, we only explore the case where the last layer contains a non-linearity, namely $f_{\theta}(x) = g(a^L)$ since BCE loss is commonly used in binary classification problems where the output is a probability of the input belonging to one of two classes. In such problems, the last layer of the model typically uses a sigmoidal activation function to ensure that the output is between $0$ and $1$, representing the probability of the input belonging to a certain class.
\begin{theorem}[BCE Loss]\label{th:bce_loss}
%In the same hypothesis of the previous theorem, and using the same notation.
%Let $\mathbf{\sigma}: \mathbb{R} \to \mathbb{R}$ be a function for which exists a Pfaffian chain $ (\mathbf{\sigma}_1, \dots, \mathbf{\sigma}_\ell)$ and $\ell + 1$ polynomials $Q$ and $P_i$, $1 \leq i \leq \ell$ of degree $\beta $ and $\alpha$ , respectively s.t.:

%$$ \frac{\partial \sigma_i(a)}{\partial a}  = P_i(a, \sigma_1(a), \dots, \sigma_i(a)), \; \; 1 \leq i \leq \ell$$
%$$ \sigma(a) = Q(\sigma_1(a), \dots, \sigma_{\ell} (a)) $$

%Moreover, let $g: \mathbb{R} \rightarrow \mathbb{R}$ be a function for which there exists a Pfaffian chain $( g_1, \dots, g_{\ell_{o}})$ and $\ell_{o} +1$ polynomials,  $Q_g$  and $P_g^i$, $1 \leq i \leq \ell$ of degree $\beta_g$ and $\alpha_g$, respectively, s.t. :

%$$ \frac{\partial g_i(a)}{\partial a}  = P_g^i(a, g_1(a), \dots, g_i(a)), \; \; 1 \leq i \leq \ell_{o}$$

%$$ g(a) = Q_g(g_1(a), \dots, g_{\ell_{o}} (a)) $$

%Let $f(\theta,x )$ be the function implemented by a neural network with parameters $\theta \in \mathbb{R}^n$ and input $x \in \mathbb{R}^{n_0}$.
%Let $f(\theta,x )$ be the function implemented by a neural network with parameters $\theta \in \mathbb{R}^n$, input $x \in \mathbb{R}^{n_0}$ and activation function $\sigma$ for all layers except  the last. The last layer can either have activation function $g$ or be linear.
%We denote $n$ as the total number of parameters in the neural network $n=\sum_{k=1}^N n_k \cdot n_{k-1}$. 
%Let us consider the case of a real output, i.e. $f : \mathbb{R}^n \rightarrow \mathbb{R}$.
Let the hypothesis of Theorem \ref{th:mse_loss} hold. If the activation function $g$ is also used in the last layer, 
the BCE Loss function is Pfaffian with format
%\textcolor{red}{$ ((L-2)( \beta + \alpha  + \alpha (\beta + 1) ) + \beta , \; \; \beta + 1, \; \;  \ell h  (L-1)  )$ when the last layer is linear and $((L-2) ( \beta + \alpha + \alpha (\beta+1)) +\beta +1, \; \; \beta , \; \;  \ell h ( L- 1) + \ell_{g}    )$}
$$ \left((L-2)(\text{degree}(\sigma') +1 ) + \text{degree}(\sigma') + \text{degree}(g')+3, 1,  m (\ell_{\sigma}\sum_{k=1}^{L-1} \;  n_k +  \ell_{g}   + 4  ) \right).$$

If the last activation function is the sigmoid function, the BCE Loss function has the format

$$ \left((L-2)(\text{degree}(\sigma') +1 ) + \text{degree}(\sigma') + 3, 1,  m (\ell_{\sigma}\sum_{k=1}^{L-1} n_k + 1 ) +1  \right).$$

\end{theorem}

%Indeed the maximum degree of the polynomial expressing the derivative is 2 and since we include the output function in the chain.

Theorems \ref{th:mse_loss} and \ref{th:bce_loss} hold in general for any Pfaffian activation function and for any sequence layer widths  $(n_0,n_1, \dots, n_L)$. 

In the following, we present the outcomes for the particular scenario in which the activation function of the intermediate layers is either a sigmoid or a hyperbolic tangent. Additionally, for the sake of simplicity, we assume that all hidden layers share the same width, denoted as $h$.

\begin{corollary}\label{cor:mse_format}
Let us consider a feedforward perceptron network where all hidden layers have the same width {$h$}. The activation function of intermediate layers can be either the hyperbolic tangent ($\tanh$)  or the sigmoid function ($\mathrm{logsig}$), and the loss function is the  Mean Squared Error. In this setting, the following holds:
\begin{itemize}
\item when the  last layer of the network is linear, the Pfaffian format of the loss function is given by  $$ \left( 3(L-2), \; 4 , \; m\left(L-1\right)h\right); $$

\item when the last layer passed through a sigmoid activation function, the Pfaffian format of the loss function is
$$\left(3(L-2) +5 , \; 2, \;   \;  m (h(L-1)+1)  \right)$$
\end{itemize}
% \note{Quale funzione di activazione non lineare si usa nello strato di uscita?}

\end{corollary}

We can state an analogous result for the BCE loss function.

\begin{corollary}\label{cor:bce_format}
For a feedforward perceptron network where all the hidden layers have the same width $ h$, trained using BCE loss function and a non-linear last layer, the following holds: 
\begin{itemize}
    \item  if the non-linearity used is the sigmoid function, the Pfaffian format of the loss function is $$\left(  3(L-2)+5, \; 1 , \; m \left((L-1)h+1\right)+1  \right)$$
    \item if the non-linearity used is $\tanh(x)$, the Pfaffian format of the loss function is
$$( 3(L-2)+7, \; 1 , \; m((L-1)h+5 )) .$$
\end{itemize}

\end{corollary}

We now state the main results of our work. 

The presented theorem investigates the dependence of the sum of Betti numbers associated with the semi-Pfaffian variety $S_{\mathcal{N}}$, which represents the parameter set where the loss is non-negative, on factors such as the number of samples, network width, and network depth.
Using Corollaries \ref{cor:mse_format} and \ref{cor:bce_format}, we can derive the bounds on the Betti numbers of the Pfaffian semi-variety for both MSE and BCE loss functions.

\begin{theorem} \label{th:bounds_sum_betti}
   Let us consider a deep feedforward perceptron network  $\mathcal{N}$ with
    $L \geq 3$ layers all having width $h$. The activation function can be either the hyperbolic tangent or the sigmoid function; the loss is either the MSE or the BCE  function, and the last layer is either linear (only for MSE) or nonlinear. Moreover, let us denote by $S$ the semi-Pfaffian variety given by the set of parameters where the loss function is  non-negative, 
   i.e. $S = S_{\mathcal{N}} = \{ \theta | \mathcal{L}(\theta) \leq c \}$ for a threshold $c \in \mathbb{R^+}$. Then, the sum of the Betti numbers
   $B(S)$ is bounded as follows.
   \begin{itemize}
       \item With respect to the number of samples $m$, fixing $h$ and $L$ as constants,| we have that $  B(S) \in \kappa_1^{O( m^2)}$, where  $\kappa_1$ is a constant greater than $2$.
       \item With respect to $h$, we have
$ B(S) \in O(h^2)^{O(h^2)}. $

\item Finally, with respect to $L$, $B(S) \in 2^{O(L^2)}O(L^2)^{O(L)}$.
   \end{itemize}

On the other hand, in the case of a shallow network with one hidden layer, i.e.,  $L = 2$, the following results hold.
   \begin{itemize}
\item With respect to $m$,  the bound is the same we obtained for deep networks, $B(S) \in  \kappa_2^{O(m^2)}$, where $\kappa_2$ is a constant greater than $2$.
 \item With respect to $h$, we have $ B(S) \in O(h)^{O(h)}. $
   \end{itemize}

\end{theorem}

It is worth emphasizing that the bounds concerning $h$ and $L$ offer an insight into the relationship between the topological complexity of the loss landscape and the total number of parameters $\tilde{n}$. Specifically, in both cases, as $\tilde{n}$ varies, by treating $h$ as a variable while keeping $L$ fixed, we explore the effect of changing the network width. Conversely, by treating $L$ as a variable while keeping the width fixed, we investigate the impact of altering the network depth.

A first major remark from Theorem \ref{th:bounds_sum_betti} is that the upper bound on the Betti numbers associated to the loss function is only exponential in the number of samples $m$, while it is superexponential in the number of neurons $h$ or in the number of layers $L$. Intuitively, the take-home message is that the topological complexity of the loss function is less conditioned by the number of samples than by the number of parameters.

As it may not appear surprising, Theorem \ref{th:bounds_sum_betti} also suggests that the complexity of the loss landscape with respect to deep networks increases with the number of neurons $h$ at a much faster pace than the one with respect to the shallow networks, going from a dependence of the type $O(h^2)^{O(h^2)}$ to a dependence of the type $O(h)^{O(h)}$. Such a difference in behavior is coherent with results present in literature \cite{li2018visualizing}, where it is proven that, when networks become sufficiently deep, neural loss landscapes quickly move from being nearly convex to being highly chaotic.

\subsection{Regularization terms and residual connections}

% \fs{Una cosa che possiamo osservare è che la nostra teoria da un upperbound sul numero di minimi.
% Oltre al fatto che non abbiamo un lower bound per cui le stime potrebbero essere maggiori dei
% valore reali, Inoltres, gli algoritmi di ottimizzazione non esplorano tutto lo spazio
% ma solo una parte in cui la complessità della funnzione potrebbe essere minore.
% Regolarizzazione e skip connection potrebbero essere meaccanismi che fanno sì che la zona in cui lavora
% l ottimizzazione sia più favorevole di qullo che suggerisce il bound globale.}
\paragraph{The role of regularization}
\begin{remark}{($\ell_2$ regularization)}
One could be interested in seeing how the introduction of regularization terms influences our analysis. We can face this new scenario in case we add an $\ell_2$ \textit{regularization term},
being the $\ell_2$ regularization term $\Omega(\theta) = \sum \limits _i \| \Bar{W}_i  \|^2$ polynomial in the parameters $\theta$.  This term only affects the term $\beta$ of the format of the Pfaffian function $\Bar{\mathcal{L}}(\theta) = \mathcal{L}(\theta) + \lambda \Omega(\theta)$, as it affects the degree of the polynomial with respect to the weights, adding a monomial term of degree $2$;  nevertheless, the bound on the Betti numbers is not affected by it. This can be easily derived from the computations carried on in the Appendix for the proof of Theorem \ref{th:bounds_sum_betti}.
The $\ell_2$ regularization does not promote sparsity (which is instead promoted by the $\ell_1$ regularization) \cite{hastie2009elements}, but it affects the magnitude of the weights. This could suggest that the regularization term may provide a scaling of the loss function, and therefore, all local minima may still be present; nevertheless, it is possible that our theoretical bounds may not be able to catch the influence of the regularization term in the loss landscape.

\end{remark}

\textbf{Residual Connections}

Skip connections or residual connections, are widely employed in neural networks to alleviate the vanishing gradient problem and enhance information flow across layers. The ResNet model popularized this architectural design \cite{he2016deep} and has since then been adopted in various network architectures. 

\begin{remark}
 Introducing a residual term at each layer in a neural network, thus creating a Residual Neural Network (ResNet), does not impact the bounds provided in our analysis. It does not affect either the number of functions required in the chain or the maximum degree of the polynomials. The addition of skip connections combines the output of a previous layer with that of a subsequent layer through summation. Regarding our analysis, the essential factors, such as the degree of polynomials and the length of the Pfaffian chain, remain the same. The only change lies in the specific terms to be included within the chain. See section \ref{appendix:residualconnections} in the Appendix for more details. Consequently, it implies that utilizing a ResNet rather than a conventional feedforward neural network does not alter the topology of the loss function or its optimization process. Instead, the primary advantage lies in enhancing the network's expressive capacity.

%If we would include a residual term $+x$ at each layer (thus designing a so-called Residual Neural Network) it wouldn't affect the computation of the derivatives, being a constant term w.r.t the parameters $\theta$. This could suggest that using a ResNet instead of a traditional feedforward neural network would not have any effect on the topology of the loss function, and in turn on its optimization; mostly, it should be a way to increase the expressive power of the network.

Skip connections allow gradients to flow more easily during backpropagation, facilitating the training of deeper networks.
It has been observed that skip connections promote flat minimizers and prevent the transition to chaotic behavior \cite{li2018visualizing}.
Our current theoretical framework is limited in capturing the reduced complexity of the loss landscape induced by skip connections. Specifically, our theory provides an upper bound on the number of minima, lacking a lower bound, which implies that our estimate may exceed the actual value.
Moreover, the analysis in \cite{li2018visualizing} focuses solely on the minima, while our proposed bound encompasses the sum of all non-zero Betti numbers.
Finally, it is worth observing that the optimization algorithms do not explore the whole space but only a part where the complexity of the function might be lower. Therefore, regularisation and skip connections could be mechanisms for which only submanifolds are explored by the optimization algorithm, and such behavior could not be captured by what the global bound suggests.

\end{remark}
\section{Conclusions}\label{sec:conclusions}

In our investigation, we determined that when employing a Pfaffian non-linearity, both the MSE and BCE loss functions can be represented as Pfaffian functions. Subsequently, we analyzed the respective Pfaffian chains obtained in each case. Specifically, we examined the differences in the complexity and performance of the Pfaffian chains resulting from the use of the two loss functions.

When studying the complexity of the loss landscape, a superexponential dependency on the network parameters has been found; interestingly, a qualitative difference can be highlighted between the shallow and the deep case, as we focus on the impact of the number of neurons $h$. Indeed, as the number of layers starts to increase, the superexponential dependency involves a term $h^2$, and not $h$ anymore. This result is aligned with the general intuition and previous works in literature \cite{bianchini2014complexity}.
In any case, the asymptotic analysis shows that the sum of Betti numbers has an exponential dependency on the square of the number of samples $m$.

It is worth underlying the characterization of the topological complexity we derived for loss functions with an additional $\ell_2$ term; from our analysis point of view, it seems that the presence of a regularization term is not implied in the design of the loss landscape, pointing out to a different role of the regularization itself in the network training, e.g. the optimization process. Nevertheless, being those boundaries not tight, space for a deeper analysis is not left out.

Bounds provided by the sum of Betti numbers are evidently not tight; our analysis suggests a qualitative interpretation, more than a quantitative one. Clearly, obtaining a bound on the number of connected components $B_0(S)$ rather than $B(S)$ would give a more accurate characterization of the topology of the loss landscape; this is a perspective to be considered for future works. 

Moreover, it would be interesting to connect our results with the \textit{mode connectivity} framework, in specific through the lens of Morse theory \cite{akhtiamov2023connectedness}, which provides a characterization of the set of configurations of a neural network with respect to a fixed empirical loss value. Indeed, improving our analysis with Morse theory could help to define the connectivity maps at many dimension levels, leading to a better quantification of each single Betti number.

\bibliographystyle{abbrv}
\bibliography{references}

\appendix
\section{Appendix}

\subsection{Making derivatives explicit using Backpropagation}

Let  $\mathcal{L(\theta),D}= \frac{1}{m} \sum_{i=1}^m  loss(f_{\theta}(x_i),y_i)$ be a generic loss.
We aim to determine the gradient for a given input-output pair $(x_i, y_i)$, with respect to the weight variables $w_{jk}^{l}$ (connecting the $j$-th neuron of layer $l-1$, with the $k$-th neuron of layer $l$), which are elements of the augmented weight matrix ${W}^l$. The gradient components $\frac{\partial \mathcal{L}}{\partial w_{jk}^{l}}$  can be calculated through the chain rule.
%However, since the impact of a weight in $\mathrm{W}^l$ on the loss is solely through its influence on the subsequent layer, there is a more efficient way of expressing it.
The Backpropagation algorithm provides an efficient method of spreading the error contribution back through the layers for updating weights.
%for calculating $\delta^l$ for each layer and linking that to the derivative of interest, $\frac{\partial \mathcal{L}}{w_{i,j}^l}$.

Let us define $\delta^l_k=\frac{\partial \mathcal{L}}{\partial a_k^l}$, for $l=1,\ldots, L$, as the derivative of the cost function with respect to the activation  $a^l_k$ of the $k$-th neuron of layer $l$. Then 
$$
\frac{\partial \mathcal{L}}{\partial w_{i,j}^l} = \delta_{j}^l z_{i}^{l-1},
$$ 
which represents a polynomial function in $ \delta_j^l, z_i^{l-1}$, so that all $ \delta_j^l$, $z_i^{l-1}$ and their derivatives belong to the Pfaffian chain describing $\mathcal{L}$.
Moreover, by the chain rule, we have that 
 \begin{align*}
    \delta^{l}_j& = \;  \frac{\partial \mathcal{L}}{\partial a_j^l} = \; \sum_{k=0}^{n_{l+1}} \;   \frac{\partial \mathcal{L}} {\partial a_k^{l+1}} \; \frac{\partial a_k^{l+1}}{\partial a_j^l} = \;  \sum_{k=0}^{n_{l+1}} \; \delta^{l+1}_{k}\; \frac{\partial a_k^{l+1}} {\partial a_j^l} =  \; \sum_{k=0}^{n_{l+1}} \; \delta^{l+1}_{k}\; w_{k,j}^{l+1} \sigma'(a_{j}^l), \\
 \end{align*} 
which means that $\delta_j^l $ is polynomial in all $n_{l+1}$, $ \delta_i^{l+1}$ and $ \sigma'(a_j^{l})$,  so that we  have also to include all $ \delta_j^{l+1}$ and all $\sigma'(a_j^{l})$ and their derivatives in the Pfaffian chain.

Summing up, fixing an input $x_i, y_i$ and proceeding backward through the layers, we can derive that 
\begin{align}
\label{degree_on_delta_l}
    \frac{\partial \mathcal{L}}{\partial w_{i,j}^l}  = \textrm{poly}(\delta^{L}_1, \dots, \delta^{L}_{n_L}, \sigma'(a_1^L), \dots, \sigma'(a_{n_L}^L),  \dots, \sigma'(a_1^{l+1}), \dots \sigma'(a_{n_{l+1}}^{l+1}), \sigma'(a_j^l), \sigma(a_i^{l-1})  ).
\end{align}

\begin{remark}
    Notice that if $\sigma $ is Pfaffian. It follows that the derivative $\sigma'$ is polynomial in the factor of the chain, and the degree of the polynomial is at most $\alpha_{\sigma}$, while the degree of $\sigma$ in its chain is  $\beta_{\sigma}$.

Consequently, being $\frac{\partial \mathcal{L}}{\partial w_{i,j}^l} = \delta_{j}^l z_{i}^{l-1}$ , we have: \begin{equation}
\begin{split}
\label{poly_derivative_L}
    \frac{\partial \mathcal{L}}{\partial w_{i,j}^l} = \textrm{poly}(\delta^{L}_1, \dots, \delta^{L}_{n_L}, \sigma_1(a_1^L), \dots, \sigma_{\ell_{\sigma}}(a_1^L), \dots, \sigma_1(a_{n_L}^L), \dots, \sigma_{\ell_{\sigma}}(a_{n_L}^L)  \dots, \sigma_1(a_1^{l+1}), \dots, \sigma_{\ell_{\sigma}}(a_1^{l+1}), \\ \dots,  \sigma_1(a_{n_{l+1}}^{l+1}), \dots,  \sigma_{\ell_\sigma}, \dots, (a_{n_{l+1}}^{l+1}), \dots, \sigma_1(a_j^l), \dots, \sigma_{\ell_{\sigma}}(a_j^l), \sigma_1(a_i^{l-1}), \dots, \sigma_{\ell_{\sigma}}(a_i^{l-1}) ).
    \end{split}
\end{equation}
\end{remark}

% \note{Questa cosa non mi torna. Cosa contiene la chain che serve a definire le Pfaffiane?
% Penso dovrebbe contenere i $\sigma(a_{n_i}^l)$, gli  $a_{n_i}^l$ e alcune funzioni ausiliarie,  ma non le derivate. Le derivate servono per definire le funzioni della chain,
% ma non stanno nella chain, altrimenti per definirle dovremo definire le derivate delle derivate. }

\subsection{Proof of Theorems \ref{th:mse_loss} and \ref{th:bce_loss}}
\subsubsection{Preliminaries}\label{app:preliminaries}

We want to prove that the MSE loss function and the BCE loss functions are Pfaffian functions with respect to the parameters of the network, in the hypothesis that the non-linearities $\sigma$ and $g$ are Pfaffian.
To do so, we need to find a Pfaffian chain so that the loss function can be written as a polynomial in that chain, and we need to compute the degree of this polynomial in the parameters and the maximum degree of the derivatives of the functions in the chain with respect to the parameters of the network. 

Notice that in this particular case of the MSE, 
\begin{align*} loss_{\text{MSE}}(f(x_i),y_i)=  \frac{1}{2}(f(\theta, x_i) -y_i  )^2.\end{align*}
meaning that if $f$ is a Pfaffian function of a given format $( \alpha_{f}, \beta_f, \ell_f)$, the loss is a Pfaffian function with respect to the same chain with format $( \alpha_f, 2 \beta_f,\ell_f)$.

In the case of the BCE loss function, 
\begin{align*} loss_{\text{BCE}}(f(x_i),y_i)= - y_i  \log(f_{\theta}(x_i)) - (1-y_i)  \log(1- f_{\theta}(x_i)).\end{align*} always assuming that  $f$ is a Pfaffian function of a given format $( \text{degree}(f')-1, \beta_f, \ell_f)$ we can consider two possible chains. 
The first one is the chain in which we add to the chain of $f$  the functions $\log(f_{\theta}(x_i))$ and $\log(1- f_{\theta}(x_i))$ and their derivatives meaning that the format of the chain becomes
$ ( \max \{\text{degree} (f') +2, \alpha_f \} ,1, \ell_f + 4 ) $, where $\text{degree}(f')$ is the degree of $f'$ as polynomial in the chain, we will specify which chain in the various cases. 
In case we have a sigmoid as the final activation function, we could consider a different chain in which we include the loss in the chain; in this case, as we will see in \ref{BCE_loss_appendix} to obtain a pfaffian chain, we also need to include the function the sigmoid of $f$, $\sigma(f(x))$ so the format becomes
$( \text{degree}(f') +2 ,1, \ell_{f} +2  )$.

To determine the degree of $f'$ we will need to compute the degree of $\sigma'$. This will be useful for all the different cases considered, so we're doing it in this section.

If $y_i = \sigma_i(a)$: 
\begin{equation*}
    \frac{d \sigma(a)}{d a} \bigg|_{a = a_h^l} =  \frac{ d Q(y_1, \dots, y_{\ell} )}{d a}\bigg|_{a = a_h^l} =  \Big(  \sum_{s= 1}^\ell \frac{\partial Q(y_1, \dots, y_\ell) }{\partial y_{s}} P_u(a, y_1, \dots, y_u) \Big) \bigg|_{
 \begin{array}{c}
      1\leq u \leq \ell \\
      y_u = \sigma_u(a_h^l) \\ 
      a = a_h^l
 \end{array} }
\end{equation*}

$\frac{\partial Q(y_1, \dots, y_\ell) }{\partial y_{s}} $ has degree $ \beta_{\sigma} - 1$ and $P_u(a, y_1, \dots, y_u)$ has degree $\alpha_{\sigma} $ in $\sigma_1(a), \dots, \sigma_{\ell}(a)$. In conclusion $  \frac{d \sigma(a)}{d a} \bigg|_{a = a_h^l}  $ is a polynomial of degree $ \beta_{\sigma}  + \alpha_{\sigma} - 1 $ if, $\forall i $, $P_i$ does not depend directly on $a$, and it is  $ \beta_{\sigma}  + \alpha_{\sigma} - 1 + \alpha_{\sigma} (\beta_{\sigma} +1) $ in the general case.

Indeed if $P_i$ depends on $a$, we have that $P(a_h^l, \sigma_1(a_h^l), \dots, a_h^l)$ has degree $\alpha_{\sigma}(\beta_{\sigma} +1)$, since the degree of $a_h^l $ is $\beta_{\sigma} +1 $ in the Pfaffian chain. 
Notice that $a_h^l= W^l \sigma(a^{l-1})$ and $\sigma(a^{l-1})$ has degree $\beta_{\sigma} $.

Summarizing : 
\begin{equation} \label{eq:degree_sigma_primo}
\text{degree}(\sigma')= \begin{cases}
    \beta_{\sigma}  + \alpha_{\sigma} - 1  \; \; \text{case } 1\\ 
   \beta_{\sigma}  + \alpha_{\sigma} - 1 + \alpha_{\sigma} (\beta_{\sigma} +1)  \text{case  }2
\end{cases}
\end{equation}
Where case $1$ refers to the case where $P^i_{\sigma}(a, \sigma_1(a), \dots , \sigma_{\ell_{\sigma}}(a))$ 
don't depend explicitly $a$ and case $2$ where they depend on it.

%We now analyze what is the $\delta^L$ in the case of the two loss functions considered in the theorem.

\subsubsection{ MSE loss function}
\begin{proof}[Proof of Theorem \ref{th:mse_loss}]

% The MSE loss function is defined as:
% \begin{align*} loss_{\text{MSE}}(f(x_i),y_i)=  \frac{1}{2}(f(\theta, x_i) -y_i  )^2.\end{align*}

In our hypothesis $n_L=1$, so $\delta^L$, $a^L$ and $f_{\theta}(x)  $ are scalars. Depending on whether the last layer is linear or the non-linearity $g$ is applied, we have that 
\begin{align*}
f_{\theta}(x) =
\begin{cases}
  a^L\\
 g(a^L).
\end{cases}
\end{align*}

\paragraph{Linear last layer}
In the case of linear activation, given that $a^{L}= W^L\sigma(a^{L-1})$ and that $\sigma$ is Pfaffian with respect to the chain $\sigma_{1}, \dots,  \sigma_{\ell_{\sigma}}$  we obtain that $f_{\theta}(x) = a^{L} $ is polynomial in the functions   following chain  \begin{equation*}
    (((\sigma_k (a_i^j))_{k=1, \dots \ell_{\sigma}})_{i= 1, \dots, n_j})_{j = 1, \dots, L-1}
\end{equation*}

The degree of the Pfaffian function $f$ in this chain is $\beta_f = \beta_{\sigma} +1 $ ; the maximum degree of the derivatives depends on the degree of $\sigma'$ in \eqref{eq:degree_sigma_primo}  and is given by the chain rule, the worst case is given by deriving of the term of the vector $\sigma(a^{L-1})$ with respect to one of the weights of the first layer. 
In this case, applying the chain rule and  going backward layer by layer, we obtain that every step, we multiply the weight of one layer and the $\sigma'$ it follows that the degree of $ \frac{\partial \sigma(a^{L-1}) }{\partial w^{1}_{i,j}}$ is $ (\text{degree}(\sigma') +1 ) (L-2) +\text{degree}(\sigma') $

% f the derivatives of the functions in the chain of $\sigma$, i.e. $ \alpha_f = \alpha_{\sigma}$.

%$\sigma_{k}(a_i^j)$ for  $j = 1, \dots, {L-1}$, $i= 1, \dots n_{j}$, $ k = 1, \dots \ell_{\sigma} $.

Taking into account what was said in Section \ref{app:preliminaries},  we obtain that for a single input point, $x$, we obtained that the MSE loss with linear activation for  the last layer is Pfaffian with respect to the following chain of length $\ell_{\sigma} \sum_{k=1}^{L-1} n_k$: 
% \begin{align}
%     (\sigma_1(a_1^{L-1}), \dots, \sigma_{\ell_{\sigma}}(a_1^{L-1}), \dots, \sigma_1(a_{n_{L-1}}^{L-1}), \dots, \sigma_{\ell_{\sigma}}(a_{n_{L-1}}^{L-1})  \dots, \sigma_1(a_1^{l+1}), \dots, \sigma_{\ell_{\sigma}}(a_1^{l+1}), \\ \dots,  \sigma_1(a_{n_{l+1}}^{l+1}), \dots,  \sigma_{\ell_\sigma}, \dots, (a_{n_{l+1}}^{l+1}), \dots, \sigma_1(a_j^l), \dots, \sigma_{\ell_{\sigma}}(a_j^l), \sigma_1(a_i^{l-1}), \dots, \sigma_{\ell_{\sigma}}(a_i^{l-1}) )
% \end{align}

\begin{equation}
\label{eq:chain_MSE_linear_appendix}
    (((\sigma_k (a_i^j))_{k=1, \dots \ell_{\sigma}})_{i= 1, \dots, n_j})_{j = 1, \dots, L-1}
\end{equation}

The order of the chain is given, going from the inner cycle to the outer cycle.

Considering that we have to consider all the input points, the length of the chain becomes $m \ell_{\sigma}\sum_{k=1}^{L-1} n_k$.
The format of the chain of the MSE loss function is therefore
\begin{equation}
    \label{eq:format_MSE_linear_appendix}
    ( (\text{degree}(\sigma') +1 ) (L-2) +\text{degree}(\sigma') , \; 2(\beta_{\sigma} +1), \; m \ell_{\sigma}\sum_{k=1}^{L-1} n_k)
\end{equation}

\paragraph{Non-linear last layer with non-linearity $g$}
In this case, we need to add to the chain described in Equation \eqref{eq:chain_MSE_linear_appendix} the terms $(g_k(a^L))_{k=1, \dots, \ell_{g}}$. 
The final chain will be:

\begin{equation}\label{eq:nonlinear_last_chain}
    [(((\sigma_k (a_i^j))_{k=1, \dots \ell_{\sigma}})_{i= 1, \dots, n_j})_{j = 1, \dots, L-1}, (g_k(a^L))_{k=1, \dots, \ell_{g}}]
\end{equation}

The degree of the function $f$ in this chain is given by $\beta_g$, 
the maximum degree of the derivatives is the degree of $ \frac{\partial g(a^{L})}{ \partial w^{1}_{i,j}}$ and is $(\text{degree}(\sigma') +1 ) (L-2) +\text{degree}(\sigma') + \text{degree}(g')+1 $.

Using the argument we used before for $\sigma'$, we have that the degree of $g'(a)$  is $ \beta_{g}  + \alpha_{g} - 1 $ if, $\forall i $, $P^i_{g}$ does not depend directly on $a$, and it is  $ \beta_{g}  + \alpha_{g} - 1 + \alpha_{g} (\beta_{g} +1) $ in the general case.
\begin{equation} \label{eq:degree_g_primo}
\text{degree}(g')= \begin{cases}
    \beta_{g}  + \alpha_{g} - 1  \; \; \text{case } 1\\ 
   \beta_{g}  + \alpha_{g} - 1 + \alpha_{g} (\beta_{g} +1)\;  \text{case  }2
\end{cases}
\end{equation}

Summarizing the format of the chain in the case of non-linear activation for the last layer is 

\begin{equation}
\label{eq:format_MSE_nonlinear_g_appendix}
((\text{degree}(\sigma') +1 ) (L-2) +\text{degree}(\sigma') + \text{degree}(g')+1, 2 \beta_g,  m (\ell_{\sigma}\sum_{k=1}^{L-1} n_k + \ell_{g} ))
\end{equation}

\end{proof}

\subsubsection{BCE loss function}\label{BCE_loss_appendix}
\begin{proof}[Proof of Theorem \ref{th:bce_loss}]
We study the two cases when the intermediate activation $g$ is the sigmoid and when it is a different function.

 \paragraph{Non-linear activation $g$ different from the sigmoid function}

In this case, the Pfaffian chain for the loss function will be the chain described in \ref{eq:nonlinear_last_chain} to which we add the following terms
$ \frac{1}{g(a^L)}, \log(g(a^L)), \frac{1}{1- g(a^L)} , \log(1-g(a^L))$. The length of the chain will be 
$\sum_{k=1}^{L-1} \;  n_k +  \ell_{g}   + 4 $. The degree of the loss function with respect to this chain is $1$, and the maximum degree of the derivatives is  given by the degree of the following derivative $\frac{\partial 1/g(a^L)}{ \partial w^{1}_{i,j}}$
that is $ (L-2)(\text{degree}{\sigma'} +1 ) + \text{degree}(\sigma') + \text{degree}(g')+3 $. 

It follows that the format of the chain for the BC loss function with sigmoid activation for the last layer is 
$$ ( (L-2)(\text{degree}(\sigma') +1 ) + \text{degree}(\sigma') + \text{degree}(g')+3, 1,  m (\ell_{\sigma}\sum_{k=1}^{L-1} \;  n_k +  \ell_{g}   + 4  )   )$$

\paragraph{Non-linear activation $g$ is the sigmoid function }
In this case, we want to include the loss in the chain and use the trick of backpropagation introduced in section \ref{app:preliminaries} to be sure that its derivatives are polynomial in the chain. 
Eq \ref{degree_on_delta_l} shows that the derivative of the loss with respect to the parameters is poly in $\delta_i^L$ and in the terms of the chain \eqref{eq:nonlinear_last_chain}, in this case with  $g$ equal to the sigmoid function.

If we consider only a single sample $x$ and the output of the network $f_{\theta}(x) = g(a^L)$. We have that 
\begin{align}
\label{delta_BCE}
\frac{\partial loss_{\text{BCE}}(y,f_{\theta}(x))}{\partial a^L} =  g'(a^L)\frac{1}{g(a^L) (1- g(a^L))}(y_i - f_{\theta}(x))  =\sum_{i=1}^m  \frac{ g'(a^{L})}{g(a^L)(1-g(a^L))}(y_i - g(a^L)).
\end{align}

If the non-linearity $g$ is the sigmoid function $ \Big( g(x) = \frac{1}{1 + e^{-x}}\Big)$, the term $ \frac{ g'(a^{L})}{g(a^L)(1-g(a^L))}$ becomes $1$, and we don't need to deal with it and the degree of $\delta^L$ is the degree of $g(a^L)$, that is $\beta_g$ that for the sigmoid function is 1. We recall that the sigmoid function is Pfaffian with format $(2,1,1)$.

It follows that the format of the chain for the BC loss function with sigmoid activation for the last layer is 
$$ ( (L-2)(\text{degree}(\sigma') +1 ) + \text{degree}(\sigma') + \text{degree}(g')+1, 1,  m (\ell_{\sigma}\sum_{k=1}^{L-1} n_k + 1 ) +1  )$$

The last $+1$ in the length is given by the fact that we're also adding the loss function computed on the input dataset to the chain. Moreover, we can compute $\text{degree}(g')$ that is this case is $2$, notice that this is smaller than the worst case proposed in \eqref{eq:degree_g_primo} that would be equal to $4$.

The final format of the chain will be 

$$ ( (L-2)(\text{degree}(\sigma') +1 ) + \text{degree}(\sigma') + 3, 1,  m (\ell_{\sigma}\sum_{k=1}^{L-1} n_k + 1 ) +1  )$$
\end{proof}

\subsection{Proof of Corollary \ref{cor:mse_format}}
\begin{proof}
    
It is enough to remark that the format of the hyperbolic tangent and the sigmoid function is $(\alpha_{\sigma}, \beta_{\sigma}, \ell_{\sigma}) = (2,1,1)$. Substituting these values in Theorem \ref{th:mse_loss} leads straightforwardly to the statement. 

\end{proof}

\subsection{Proof of Corollary \ref{cor:bce_format}}
\begin{proof}
The result for the hyperbolic tangent activation function can be obtained by applying Theorem \ref{th:bce_loss} with its corresponding format of $(2,1,1)$.
On the other hand, for the sigmoid function, we have that  $\sigma'(x) = \sigma(x) (1-\sigma(x))$. This allows us to obtain a linear dependency on $\sigma(x)$ in the derivative of the BCE loss function. Indeed, 
\begin{align*}
\frac{\partial \mathcal{L}(f)}{\partial \theta} & = - \frac{\partial}{\partial \theta}(y \log(\sigma(f(\theta))) + (1-y) \log(1- \sigma(f(\theta)))) \\
& =   -[ \frac{y}{\sigma(f(\theta))} \sigma'(f(\theta)) \frac{\partial f(\theta)}{\partial \theta} - \frac{1-y}{1-\sigma(f(\theta))} \sigma'(f(\theta))\frac{\partial f(\theta)}{\partial \theta}] \\
&= -y (1-\sigma(f(\theta))) \frac{\partial f(\theta)}{\partial \theta} + (1-y) \sigma(f(\theta))\frac{\partial f(\theta)}{\partial \theta}  \\
&=  (\sigma(f(\theta))-y)\frac{\partial f(\theta)}{\partial \theta}
\end{align*}
\end{proof}

\subsection{Proof of Theorem \ref{th:bounds_sum_betti}}

\begin{proof}

We can use Theorem \ref{Teo:8} with $U= \mathbb{R}^{\tilde{n}}$, with $\tilde{n}= h( n_0 + 1) + h(h+1)(L-2) + h+1 = h^2 (L-2) + h(n_0 + L) +1 $ the total number of parameters of the network.

In this case the term $s^{n'}$ in Equation \eqref{eq:Bettinumber}
can be ignored since $s=1$.

\paragraph{Bounds for MSE loss function: deep case}

For $L\geq3$:
\begin{itemize}
\item if the last layer has a linear activation, we have that

   \begin{equation} \label{eq:mse_deep_linear}
       B(S) \in  2^{[m(L-1)h(m(L-1)h-1)]/2}O( f(n_0,h,L,m)^{h^2(L-2) + h( L+ n_0 + m(L-1)) +1})
   \end{equation}
   with $f(n_0,h,L,m) = 4[h^2(L-2) + h( L+  n_0)+ 1]+ 3(L-2) \cdot \min(h^2(L-2) + h( L+  n_0)+ 1,m(L-1)h)$

If $m>>h$ it becomes
$$ B(S) \in  2^{(m(L-1)h(m(L-1)h-1))/2}O(4[h^2(L-2) + h( L+  n_0)+ 1]+ 3(L-2)(h^2(L-2) + h( L+  n_0)+ 1))^{h^2(L-2) + h( L+ n_0 + m(L-1)) +1}) $$

If $h>>m$, %we are in overparametrized regime,
and we consider $L$ and $m$ as constant it becomes
$$ B(S) \in  2^{(m(L-1)h(m(L-1)h))/2}O(4[h^2(L-2) + h( L+  n_0)+ 1] 3(L-2)hm(L-1))^{h^2(L-2) + h( L+ n_0 + m(L-1)) +1} .$$
It is important to note that the overparametrized regime falls within this scenario.

\item if the last layer has a nonlinear activation, we have that: 

\begin{equation}\label{eq:mse_deep_nonlinear}   
B(S) \in  2^{[(m(h(L-1)+1))(m(h(L-1)+1)+1)]/2}O( g(n_0,h,L,m))^{h^2(L-2) + h( L+ n_0 + m(L-1)) +2}
\end{equation}

 with $g(n_0,h,L,m) = 2[h^2(L-2) + h( L+  n_0)+ 1]+ (3(L-2)+5) \cdot \min(h^2(L-2) + h( L+  n_0)+ 1,m(h(L-1)+1)$ 

If $m>>h$ it becomes
\begin{align*}
 B(S) \in  2^{[(m(h(L-1)+1))(m(h(L-1)+1)+1)]/2}O\big( 2[h^2(L-2) + h( L+  n_0)+ 1]+& \\+(3(L-2)+5)(h^2(L-2) + h( L+  n_0)+ 1\big)^{h^2(L-2) + h( L+ n_0 + m(L-1)) +2}&
 \end{align*}

If $h>>m$, %we are in overparametrized regime,
and we consider $L$ and $m$ as constant it becomes

\begin{align*} 
B(S) \in  2^{[(m(h(L-1)+1))(m(h(L-1)+1)+1)]/2}O \big( 2[h^2(L-2) + h( L+  n_0)+ 1]+ & \\
+(3(L-2)+5)m(h(L-1)+1)\big)^{h^2(L-2) + h( L+ n_0 + m(L-1)) +2} &
\end{align*}
\end{itemize}

In both cases, we can see that, as a function of the number of samples $m$, fixing $h$ and $L$ as constants, we have that $  B(S) \in c^{O( m^2)}$, where  $c$ is a constant greater than $2$.
As a function of $h$, considering the other variables as constants
$ B(S) \in O(h^2)^{O(h^2)}. $
Eventually, as a function of $L$, considering $m$ and $h$ as constants, 
$  B(S) \in 2^{O(L^2)}O(L^2)^{O(L)}.$

\paragraph{Bounds for MSE loss function: shallow case}

In the case of $L=2$, all the terms in which $L-2$ occurs vanish. Therefore,

\begin{itemize}
    \item if the last layer has a linear activation, equation \ref{eq:mse_deep_linear} simplifies in the following way:
$$  B(S) \in  2^{[m h(mh-1)]/2}O( [ 4 h(2 + n_0)+  4 ]^{ h( 2+ n_0 + m) +1})$$

\item if the last layer has a nonlinear activation since the number of samples is usually larger than the input dimension, $ \min(h(2+n_0)+1, 2mh) = h(2+n_0)+1$. This simplify equation \ref{eq:mse_deep_nonlinear} in the following way:
$$  B(S) \in  2^{[2mh(2mh-1)]/2}O( [ 9 h(2 + n_0)+  9 ]^{ h( 2+ n_0 + 2m) +1})$$

\end{itemize}
As a function of $m$  the results are the same as we obtained for deep networks, $B(S) \in  c^{O(m^2)}$.
Conversely, as a function of $h$, the dependency changes and we obtain 
$ B(S) \in O(h)^{O(h)}. $

\paragraph{Bounds for BCE loss function: deep case} For $L \geq 3$:
\begin{itemize}
    \item if the last layer has a sigmoid activation function, we have that 
   \begin{equation}\label{eq:bce_deep_sigmoid}
\begin{split}
     2^{[(m((L-1)h+1)+1)(m((L-1)h+1))]/2}O(g(n_0,h,m,L)^{h(m(L-1) + n_0 +2 + (h+1)(L-2))+m+2})
\end{split}
\end{equation}
with $ g(n_0,h,m,L) = h^2 (L-2) + h(n_0 + L) +1  + [3(L-2)+5]\min(h^2 (L-2) + h(n_0 + L) +1, [m ((L - 1)h + 1) + 1] )$

If $m>>h$  $g(n_0,h,m,L)$ becomes 
$$ g(n_0, h,m, L ) =  h^2 (L-2) + h(n_0 + L) +1  + [3(L-2)+5](h^2 (L-2) + h(n_0 + L) +1)$$

On the other side, if $h>>m$ and we consider $m $ and $L$ as constant, we have that
 $$ g(n_0,h,m,L) = h^2 (L-2) + h(n_0 + L) +1  + [3(L-2)+5][m ((L - 1)h + 1) + 1]. $$
\end{itemize}

\paragraph{Bounds for BCE loss function: shallow case} For $L=2$:
\begin{itemize}
    \item if the last layer has a non-linear activation function, we have that 
   \begin{equation}\label{eq:bce_deep_sigmoid}
\begin{split}
     2^{[(m(h+1)+1)(m(h+1))]/2}O(g(n_0,h,m)^{h(m + n_0 +2 )+m+2})
\end{split}
\end{equation}
with $ g(n_0,h,m) =  h(n_0 + 2) +1  + 5\min( h(n_0 + 2) +1,m(h+1)+1 )$

If $m>>h$,  $g(n_0,h,m,L)$ becomes 
$$ g(n_0, h,m ) =  h(n_0 + 2) +1  + 5( h(n_0 + 2) +1).$$

On the other side, if $h>>m$ and we consider $m $ and $L$ as constant, we have that
 $$ g(n_0,h,m,L) =   h(n_0 + 2) +1  + 5[m(h+1)+1]. $$
\end{itemize}

% Similarly for the BCE loss, if the non-linearity is given by the sigmoid function, we have that $\ell = m(L+2)h$, $\alpha = 3L-1$ and $\beta  = 1$ ,
% so we obtain that
% \begin{equation}\label{eq:bce_deep_sigmoid}
% \begin{split}
%      2^{(m(L+2)h(m(L+2)h-1))/2}O(g(n_0,h,m,L)^{h(m(L+2) + n_0 +2 + (h+1)(L-2))+1})
% \end{split}
% \end{equation}
% with $ g(n_0,h,m,L) = h^2 (L-2) + h(n_0 + L) +1  + \min(h^2 (L-2) + h(n_0 + L) +1,m(L+2)h )(3L-1)$

% If $m>>h$  $g(n_0,h,m,L)$ becomes 
% $$ g(n_0, h,m, L ) =  h^2 (L-2) + h(n_0 + L) +1  + (h^2 (L-2) + h(n_0 + L) +1)(3L-1).$$

% On the other side, if $h>>m$ and we consider $m $ and $L$ as constant, we have that
%  $$ g(n_0,h,m,L) = h^2 (L-2) + h(n_0 + L) +1  + m(L+2)h(3L-1). $$

Resulting in asymptotic bounds equal to those derived previously for the MSE loss with non-linear activation. The same holds with similar computations in case the last activation function is the hyperbolic tangent. 

We remark that for the BCE loss, our focus is primarily on studying the case where the last layer of the neural network has a non-linear activation function since it is commonly used for binary classification tasks.

\end{proof}

\subsection{Residual Connections} \label{appendix:residualconnections}
Without loss of generality, we can consider the case in which we utilize skip connections that provide the previous layer's output, through summation, as an additional input to the subsequent layer.
In this case, denoting by $z_i$ the output of the $i$-th layer, if we consider the case, we have that 
\begin{align*}
    &z_1 = \sigma(W^1x) \\
    &z_2  = \sigma(W^2z_1) + z_1 \\
    &z_3  = \sigma(W^3z_2) + z_2 \\
    &\vdots \\
    &z_{l}  = \sigma(W^lz_{l-1}) + z_{l-1}\\
    &\vdots \\
\end{align*}

If we want to obtain the derivative of $z_{l}$ with respect to a parameter of the network $ w^k = W^{k}_i$ with $k<l$ and $i \in \{  1, \dots, n_k\}$, we have that

\begin{align*}
    \frac{\partial z_{l}}{\partial w^k} = W^{l} \sigma'(W^l z_{l-1}) \frac{\partial z_{l}-1}{\partial w^k} + \frac{\partial z_{l}-1}{\partial w^k}
\end{align*}

We observe that the derivative 	$ \frac{\partial z_{l}-1}{\partial w^k}$ appears twice in the equation. However, due to the multiplication with other terms in the first term, it only affects the degree of the polynomial in the first term of the sum. Therefore, we can disregard the second term in the equation when interested in computing the format of the Pfaffian chain.
This reasoning can be extended to all layers, demonstrating that the degrees of the derivatives will remain unchanged, just as we computed for simple feedforward perceptron neural networks. Similarly, the length of the chain remains unaltered. In this case, as well, the functions to be included in the chain are the outputs of the layers, which are exactly the same in number as those in feedforward networks, albeit with different forms.

\end{document}